%% file: main.tex
\documentclass[11pt,fleqn]{article}

\usepackage[utf8]{inputenc}
\usepackage{fullpage}

\usepackage{xcolor}
\usepackage{algorithm}
\usepackage[noend]{algpseudocode}

\usepackage{amsmath}
\usepackage{amsthm}
\usepackage{amsfonts}
\usepackage{graphicx}
\usepackage{amssymb}
\usepackage{mathtools}
\usepackage{enumerate}

\newtheorem{theorem}{Theorem}[section]
\newtheorem{lemma}[theorem]{Lemma}
\newtheorem{proposition}[theorem]{Proposition}
\newtheorem{corollary}[theorem]{Corollary}

\theoremstyle{definition}
\newtheorem{definition}[theorem]{Definition}
\newtheorem{example}[theorem]{Example}

\theoremstyle{remark}

\makeatletter
\renewcommand{\maketitle}{\bgroup\setlength{\parindent}{0pt}
\begin{flushleft}
  \LARGE \textbf{\@title} \linebreak
  
  \normalsize \@author \linebreak
\end{flushleft}\egroup
}
\makeatother

\usepackage{authblk}

\title{\Large Mathematical conjecture generation using machine intelligence}

\author[1]{Challenger Mishra}
\author[2]{Subhayan Roy Moulik}
\author[3]{Rahul Sarkar}

\affil[1]{\small The Computer Laboratory, University of Cambridge,Cambridge, CB3 0FD} 
\affil[2]{\small DAMTP, Centre for Mathematical Sciences, University of Cambridge, Cambridge, CB3 0WA} 
\affil[3]{\small Institute for Computational and Mathematical Engineering, Stanford University, Stanford, CA 94305}
\affil[ ]{}
\affil[ ]{\scriptsize cm2099@cam.ac.uk, sr2068@cam.ac.uk, rsarkar@stanford.edu}
\begin{document}

\maketitle

\begin{abstract}
Conjectures have historically played an important role in the development of pure mathematics. We propose a systematic approach to finding abstract patterns in mathematical data, in order to generate conjectures about mathematical inequalities, using machine intelligence. We focus on strict inequalities of type $f<g$ and associate them with a vector space. By geometerising this space, which we refer to as a \textit{conjecture space}, we prove that this space is isomorphic to a Banach manifold. We develop a structural understanding of this \textit{conjecture space} by studying linear automorphisms of this manifold and show that this space admits several free group actions. Based on these insights, we propose an algorithmic pipeline to generate novel conjectures using geometric gradient descent, where the metric is informed by the invariances of the \textit{conjecture space}. As proof of concept, we give a toy algorithm to generate novel conjectures about the prime counting function and diameters of Cayley graphs of non-abelian simple groups. We also report private communications with colleagues in which some conjectures were proved, and highlight that some conjectures generated using this procedure are still unproven. Finally, we propose a pipeline of mathematical discovery in this space and highlight the importance of domain expertise in this pipeline.  
\end{abstract}

\section{Introduction}

Conjectures hold a special status in mathematics. Good conjectures epitomise milestones in the pipeline of mathematical discovery, and have historically inspired new mathematics and shaped progress in theoretical physics. At the turn of the last century, David Hilbert put forward a list of 23 problems which have since driven development in geometry, number theory, and other domains of pure mathematics. 
One such conjecture is the thirteenth problem, which ultimately led to the development of a formidable representation theorem for real continuous multivariate functions in the 1950’s (due to Kolmogorov and Arnold). Not only did this posit a new vision of geometry, decades later this helped establish neurocomputing on firmer mathematical footing~\cite{H87}. Another profound example was a key observation by McKay that led to monstrous moonshine, a phenomenon connecting two seemingly disjointed parts of mathematics, namely, monster groups and modular forms. 
Hilbert's lecture at the International Congress of Mathematics in Paris, 
 exemplified the significance of \emph{identifying noteworthy problems} by making mathematical conjectures and their impact on developing new mathematics. The spirit of proposing open problems and identifying \emph{non-trivial} conjectures continues to be a rewarding and a common practice endorsed by practitioners and institutions dedicated to furthering mathematics \cite{D19}. Other notable lists of problems have been those collected by Landau \cite{Guy04}, Weil \cite{W49}, Thurston \cite{T82}, and the Clay Mathematics Institute \cite{CMI00}. 

\noindent Conjectures are unproven propositions. Formulating meaningful conjectures can be nontrivial; albeit one can be aided by  discovering new patterns and formulating well defined closed form expressions. 
 The Birch and Swinnerton-Dyer conjecture \cite{BS63,BS65}, an unsolved Millennium Prize Problem, was one of the early examples of computer assisted conjecture generation, which was proposed based on a numerical tests performed on large data in 1960s, driven by intuitions and expert insights. It is generally known that computers are canonically good at pattern recognition and processing large volumes of data algorithmically. Building on machine intelligence and domain expertise, notable recent strides were made in topics in knot theory \cite{Dothers21}, in finding formulae to equate fundamental constants as continued fractions \cite{Rothers21}, improving algorithms for matrix multiplication \cite{Fother22} and sorting \cite{Mothers23}, among a host of other applications. A common practice in these approaches has been to have interactions with domain experts for a machine guided discovery.

 \noindent Mathematical inequalities express relations of the form $f<g$.  They are ubiquitously studied across analysis, combinatorics, geometry, and so on, and are essential to bounding quantities of interest across pure and applied natural sciences. Consequently great efforts go into finding new inequalities and proving them \cite{HLP52, BB65, MV70}. 

 \noindent Bounds on prime gaps, ground state energies in physical systems, are merely a couple of examples that highlight the importance of inequalities in mathematics and natural sciences. Further, nontrivial inequalities can have a range of practical real world applications in a variety of scientific domains, as well as in interdisciplinary endeavours such as machine learning. 

 \noindent In this work, we consider the task of formulating conjectures about inequalities. In Section~\ref{sec:inequalities}, we define the notion of a \textit{conjecture space}. We develop an understanding of the structure of this space through geometerisation, and studying its linear automorphisms. We also develop an understanding of when such spaces admit quotients by free group actions. Using these insights we propose an \textit{oracle} in Section~\ref{sec:computation}, which generates conjectures by treating the problem of finding conjectures as a sampling problem. This outlines a machine-learning-guided framework to infer conjectures about  strict mathematical inequalities. The core of the oracle is geometric gradient descent, where symmetries of the underlying \textit{conjecture space} is taken into account. We make modest strides towards understanding the symmetries of the \textit{conjecture space} and pose some group theoretical questions about the underlying structure; albeit further work is required to fully exploit the advantage symmetries bring to an underlying search problem. In Section~\ref{sec:case_studies}, we utilise a basic non-geometric version of the \textit{oracle} to generate conjectures involving the prime counting function,  and diameters of Cayley graphs of non-Abelian simple groups, as proof-of-concept. Section~\ref{sec:discussion} outlines the further directions and present limitations of this approach.

\input{space_conjectures.tex}

\input{computation.tex}
\input{case_studies.tex}
\section{Discussion and outlook}\label{sec:discussion}
Our work highlights the potential for AI-driven conjecture generation. The initial findings in this work indicate that it is possible to generate nontrivial algebraic conjectures in mathematics, which often have relationships to existing long-standing conjectures, such as those given in the case studies inspired by the Hardy-Littlewood conjectures (involving the prime-counting function), and Babai’s conjecture (on diameters of Cayley graphs of non-abelian simple groups). We expect to study properties of Hermitian matrices such as Wigner matrices and local Hamiltonians (with applications to machine learning and quantum theory) in the future.  

\noindent One wishful application of such AI-driven inequality-discovery is to find elusive algebraic relations between computationally hard-to-compute quantities, $f$ (such as minimum eigenvalue of some matrix), and computationally efficient-to-compute quantities, $g_1,g_2$ (such as Trace of functions of the matrix). Being able to closely bound such hard quantities from above and below, $g_1 < f < g_2$, would then lead to discovering new algorithms to efficiently approximate a computationally hard quantity. 

\noindent Another open direction is to adapt this framework for mathematics education; wherein the \emph{oracle} outputs non-trivial conjectures which can be used fruitfully to generate novel exercise problems and can be used as a pedagogical tool to offer new practice problems for mathematics education. 

\noindent These possibilities motivate a principled study of the space of classes of conjectures, understanding their structure, and geometrisability, which we undertook in this work. Giving structure to this space through symmetries, reduces the computational difficulty of the underlying optimisation problem, and opens up the space to probe using tools from geometric analysis such as Ricci Flow. 

\noindent Although we take modest theoretical strides in this endeavour, further work remains to understand invariants of free group actions on the conjecture space. Further, the search problems we have devised through case studies in Section~\ref{sec:case_studies} are typically low dimensional optimisation problems. Including nontrivial machine architectures (such as neural networks) to capture the latent space would result in formulating more difficult optimisation problems. As such, it would be beneficial to realise the \textit{oracle} in the proposed geometric setting in future implementations. This would require us to address some group and invariant theoretic questions we have raised in this work. The machine learning framework of Gaussian processes is a natural way to sample functions from distributions, with the added constraint $f<g$. With further understanding of the conjecture space, and the distribution of conjectures, a machine guided approach facilitated by Gaussian processes could be insightful. 

\noindent In this work, we have effectively exploited symbolic machine learning through our representations and parameterisations. When considering different representations, further work is required to establish interactions of our conjecture-generation pipeline with formal proof assistants which are built on higher-order logic. In light of substantial reasoning abilities of natural language, it seems worthwhile to explore natural language representations in the future.

\noindent Impactful conjectures have certain commonalities. They are nontrivial with potentially substantial evidence in favour of it. Typically their description is terse, although this is a function of the underlying mathematical representation. Such conjectures can also be a gateway to unlocking new theorems; exemplified by the Riemann Hypothesis. We envision utilising these insights accumulated in~\cite{D19} to make more meaningful conjectures. In light of the outcomes in this paper, the prospect of exploiting machine learning in exploring conjecture spaces in a geometric fashion, with an aim to generating impactful conjectures is tantalising. This would also benefit from a quantum backend. As such, a truly interdisciplinary approach which also involves domain experts from different branches of mathematics is very much at the backbone of this proposed pipeline. 
\section*{Acknowledgements}
We acknowledge helpful discussions with Ann Copestake, Kevin Buzzard, Carl Henrik Ek, Thomas Fink, Yang-Hui He, Fabian Huch, Ferenc Husz\'ar, Vishnu Jejjala, Minhyong Kim, Oleksandr Kosyak, Neil Lawrence, Alexander Ochirov, Ilya Shkredov, and Yiannis Stathopoulos. The preprint also benefited from helpful discussions at the workshops on \textit{AI-assisted mathematical discovery} hosted at the London Institute for Mathematical Sciences, and \textit{Machine Learning for Science: Mathematics at the Interface of Data-driven and Mechanistic Modelling} hosted at Mathematisches Forschungsinstitut Oberwolfach. 

\bibliographystyle{alpha}
\bibliography{bibliography}
\end{document}

%% file: space_conjectures.tex
\section{The structure of inequalities}\label{sec:inequalities}
\subsection{Space of inequalities}\label{sec:conj-space}

Let $\mathcal{F}$ denote the field $\mathbb{R}$, and suppose $\mathcal{D}$ and $\Theta$ are finite dimensional vector spaces over $\mathcal{F}$. We assume that $\mathcal{D}$ has dimension $n$, and $\Theta$ has dimension $m$. Let $\widehat{\mathcal{D}}\subseteq\mathcal{D}$ be a compact subset, and let $\widehat{\Theta} \subseteq \Theta$ be an open, connected subset with compact closure (will check this assumption as the paper develops). We refer to $\widehat{\mathcal{D}}$ as the \textit{feature space}, and $\widehat{\Theta}$ as the \textit{latent space}. Let $C(\widehat{\mathcal{D}})$ denote the Banach space of continuous, real valued functions on $\widehat{\mathcal{D}}$ equipped with the supremum norm, i.e. if $f \in \widehat{\mathcal{D}}$, then
\begin{equation}
    \lVert f \rVert_{C(\widehat{\mathcal{D}})} := \sup_{x \in \widehat{\mathcal{D}}} |f(x)|.
\end{equation}
We define a \textit{space of relations} $\mathcal{R} := C(\widehat{\mathcal{D}}) \times C(\widehat{\mathcal{D}})$, which is again a Banach space with the norm
\begin{equation}
    \lVert (f,g) \rVert_{\mathcal{R}} := \lVert f \rVert_{C(\widehat{\mathcal{D}})} + \lVert g \rVert_{C(\widehat{\mathcal{D}})}, \;\; (f,g) \in \mathcal{R}.
\end{equation}

\begin{definition}
\label{def:conj}
Let $(f,g) \in \mathcal{R}$. We say that the tuple $(f,g)$ is a \textit{conjecture} if and only if $f(x) < g(x)$ for all $x \in \widehat{\mathcal{D}}$. The set $\mathcal{C} := \{(f,g) \in \mathcal{R}: f(x) < g(x), \; \forall x \in \widehat{\mathcal{D}}\}$ is called the \textit{space of conjectures}. 
\end{definition}
\noindent It then immediately follows that the space of conjectures $\mathcal{C}$ is a Banach manifold (see \cite{lang2012differential} for definitions):
\begin{proposition}
\label{prop:manifold-C}
$\mathcal{C}$ is an open subset of $\mathcal{R}$ and is a Banach manifold.
\end{proposition}

\begin{proof}
We know that $\mathcal{R}$ is a Banach manifold as it is a Banach space. We only need to show that $\mathcal{C}$ is open in $\mathcal{R}$, as it then implies that $\mathcal{C}$ is also a Banach manifold. To prove that $\mathcal{C}$ is open, let $(f,g) \in \mathcal{C}$, and define
\begin{equation}
    \delta := \sup_{x \in \widehat{\mathcal{D}}} f(x) - g(x).
\end{equation}
Since $\widehat{\mathcal{D}}$ is compact, we know that $f-g$ attains $\delta$ in $\widehat{\mathcal{D}}$, and since $f(x) < g(x)$ for all $x \in \widehat{\mathcal{D}}$, it follows that $\delta < 0$. Now consider the open set $U := \{(\bar{f}, \bar{g}) \in \mathcal{R}: \lVert (\bar{f}, \bar{g}) - (f,g) \rVert_{\mathcal{R}} < |\delta| / 2 \}$, which contains $(f,g)$. We will show that $U \subseteq \mathcal{C}$, which will prove that $\mathcal{C}$ is open. To see this, pick any $(\bar{f}, \bar{g}) \in U$, and let $x \in \widehat{\mathcal{D}}$ be arbitrary. We then have
\begin{equation}
\begin{split}
    \bar{f}(x) - \bar{g}(x)& = \bar{f}(x) - f(x) - (\bar{g}(x) - g(x)) + f(x) - g(x) \\
    & \leq |\bar{f}(x) - f(x)| + |\bar{g}(x) - g(x)| + f(x) - g(x) \\
    & \leq \lVert (\bar{f}, \bar{g}) - (f,g) \rVert_{\mathcal{R}} + \delta \leq \delta / 2 < 0.
\end{split}
\end{equation}
\end{proof}

\subsection{Group actions.}
\label{ssec:group-action}

The space of conjectures $\mathcal{C}$ enjoys certain nice geometrical properties. Let $\text{GL}(2,\mathbb{R})$ denote the set of all invertible $2 \times 2$ real valued matrices. We will denote $I := \left(\begin{smallmatrix} 1 & 0 \\ 0 & 1 \end{smallmatrix}\right)$. If $A \in \text{GL}(2,\mathbb{R})$ has the property that $A(f,g) \in \mathcal{C}$ for every $(f,g) \in \mathcal{C}$, we say that $A$ leaves $\mathcal{C}$ \textit{invariant}. In particular, we are interested in all the elements of $\text{GL}(2,\mathbb{R})$ that leave $\mathcal{C}$ invariant. For every $\left(
    \begin{smallmatrix} 
        A_{11} & A_{12} \\
        A_{21} & A_{22}
    \end{smallmatrix}
    \right) =: A \in \text{GL}(2,\mathbb{R})$, we first define the linear map
\begin{equation}
\label{eq:A-action}
\begin{split}
    & A : \mathcal{R} \rightarrow \mathcal{R}, \;\;\; (f,g) \mapsto (\bar{f}, \bar{g}), \\
    & 
    \begin{pmatrix}
        \bar{f} \\
        \bar{g}
    \end{pmatrix} := 
    \begin{pmatrix}
        A_{11} & A_{12} \\
        A_{21} & A_{22}
    \end{pmatrix} \;
    \begin{pmatrix}
        f \\
        g
    \end{pmatrix}.
\end{split}
\end{equation}
One can check that this defines a group action on $\mathcal{R}$
\begin{equation}
\label{eq:group-action}
\digamma: \text{GL}(2,\mathbb{R}) \times \mathcal{R} \rightarrow \mathcal{R}, \;\; (A, (f,g)) \mapsto A(f,g).
\end{equation}
Moreover every $A \in \text{GL}(2,\mathbb{R})$ defines a homeomorphism on $\mathcal{R}$, which we record as a proposition for later use:
\begin{proposition}
\label{prop:homeo-R}
Every $A \in \text{GL}(2,\mathbb{R})$ defines a homeomorphism $A: \mathcal{R} \rightarrow \mathcal{R}$ defined by Eq.~\eqref{eq:A-action}.
\end{proposition}
\begin{proof}
We need to show that (i) $A$ is continuous and bijective, and (ii) $A$ has a continuous inverse. Since $A$ is linear, (ii) follows from (i) by the open mapping theorem; so we will only prove (i). Pick $(f,g) \in \mathcal{R}$, and let $(\bar{f}, \bar{g}) := A(f,g)$. Then $\lVert (\bar{f},\bar{g}) \rVert_{\mathcal{R}} \leq 2 \lVert A \rVert_{F} \lVert (f,g) \rVert_{\mathcal{R}}$, where $\lVert A \rVert_{F}$ is the Frobenius norm of $A$, and this shows that $A$ is a continuous map. To show injectivity of $A$, if $A(f,g) = (0,0)$, it then implies that $(f,g)=(0,0)$ as $A$ is an invertible matrix. For surjectivity, we simply note that if $(f,g) \in \mathcal{R}$, then $(f,g) = A(A^{-1}(f,g))$.
\end{proof}

\noindent However, not every $A \in \text{GL}(2,\mathbb{R})$ leaves $\mathcal{C}$ invariant, for example $A = -I$ does not leave $\mathcal{C}$ invariant, while $A = I$ does. The latter example can be improved to find two sets of non-trivial subgroups of $\text{GL}(2,\mathbb{R})$, both of which leave $\mathcal{C}$ invariant:

\begin{example}[Dilations]
\label{ex:group1}
Consider the group of positive dilations $\mathcal{T} := \{\lambda \left(\begin{smallmatrix} 1 & 0 \\ 0 & 1 \end{smallmatrix}\right): \mathbb{R} \ni \lambda > 0\}$. Then $\mathcal{T}$ is a subgroup of $\text{GL}(2,\mathbb{R})$, and if $A \in \mathcal{T}$, it is clear that $A(f,g) \in \mathcal{C}$ if and only if $(f,g) \in \mathcal{C}$.
\end{example}

\begin{example}
\label{ex:group2}
Consider the $2$-parameter set $\mathcal{H}$ defined by
\begin{equation}
\label{eq:grp2}
\mathcal{H} := \left \{ \begin{pmatrix}
    p & q-1 \\
    p-1 & q
\end{pmatrix}
, \; p,q \in \mathbb{R}, \; p+q \ne 1
\right\}.
\end{equation}
Then one may check that $\mathcal{H}$ is a subgroup of $\text{GL}(2,\mathbb{R})$, and moreover every element of $\mathcal{H}$ leaves $\mathcal{C}$ invariant. First observe that if $A := \left(\begin{smallmatrix} p & q-1 \\ p-1 & q \end{smallmatrix}\right)$, not necessarily in $\mathcal{H}$, then $1 + \det(A) = \text{tr}(A)$. Moreover if $A \in \mathcal{H}$, then $A^{-1} = \frac{1}{p+q-1} \left( \begin{smallmatrix} q & 1-q \\ 1-p & p \end{smallmatrix} \right) = \left( \begin{smallmatrix} \bar{p} & \bar{q}-1 \\ \bar{p}-1 & \bar{q} \end{smallmatrix} \right)$, with $\bar{p} = \frac{q}{p+q-1}$, and $\bar{q}=\frac{p}{p+q-1}$. Thus $\det(A^{-1}) = \text{tr}(A^{-1})-1 = \bar{p} + \bar{q} - 1 = \frac{1}{p+q-1} \neq 0$. This shows $\mathcal{H}$ is closed under inverses. Next, if $A:=  \left(\begin{smallmatrix} p & q-1 \\ p-1 & q \end{smallmatrix}\right) \in \mathcal{H}$, and $B:=  \left(\begin{smallmatrix} \bar{p} & \bar{q}-1 \\ \bar{p}-1 & \bar{q} \end{smallmatrix}\right) \in \mathcal{H}$, then $AB =  \left(\begin{smallmatrix} p' & q'-1 \\ p'-1 & q' \end{smallmatrix}\right)$, where $p' = p \bar{p} + (q-1) (\bar{p}-1)$, and $q' = q\bar{q} + (p-1)(\bar{q}-1)$, and thus $\det(AB) = \text{tr}(AB)-1=p'+q'-1$. Noting that $p'+q' - 1= (p+q-1)(\bar{p} + \bar{q} - 1) = \det(A) \det(B) \neq 0$, we can conclude that $AB \in \mathcal{H}$. This proves that $\mathcal{H}$ is a subgroup. To show that $\mathcal{H}$ leaves $\mathcal{C}$ invariant, take any $(f,g) \in \mathcal{C}$, and $A:=  \left(\begin{smallmatrix} p & q-1 \\ p-1 & q \end{smallmatrix}\right) \in \mathcal{H}$. Then $A(f,g) = (pf + qg - g, pf + qg -f) =: (\bar{f}, \bar{g})$, and thus $\bar{f} - \bar{g} = f - g$. This implies that $(\bar{f},\bar{g}) \in \mathcal{C}$.
\end{example}

\noindent The above examples beg the following question: what is the largest subgroup of $\text{GL}(2,\mathbb{R})$ that leaves $\mathcal{C}$ invariant? We now proceed to show that the two groups $\mathcal{T}$ and $\mathcal{H}$ from Examples~\ref{ex:group1} and \ref{ex:group2} suffice to generate everything --- that is if $A \in \text{GL}(2,\mathbb{R})$ leaves $\mathcal{C}$ invariant, then one has $A = BC$, where $B \in \mathcal{T}$ and $C \in \mathcal{H}$. The precise statement of this fact is Corollary~\ref{cor:group-eqiv}. The key step towards this goal is the next lemma, for which we define the set $\mathcal{G} :=  \{ A \in \text{GL}(2,\mathbb{R}) : A(f,g) \in \mathcal{C}, \; \forall (f,g) \in \mathcal{C} \}$.

\begin{lemma}
\label{lem:subgroup-G}
The set $\mathcal{G}$ satisfies the following properties:
\begin{enumerate}[(a)]
    \item Every element $A \in \mathcal{G}$ can be expressed in the form
    \begin{equation}
    \label{eq:A-decomp}
        A = 
        \begin{pmatrix}
            1 & -1 \\
            1 & 1
        \end{pmatrix}
        \begin{pmatrix}
            a & c \\
            0 & b
        \end{pmatrix}
        \begin{pmatrix}
            1 & 1\\
            -1 & 1
        \end{pmatrix}, \;\;\; a,b,c \in \mathbb{R}, \; b > 0,\; a \neq 0.
    \end{equation}
    \item $\mathcal{G}$ is a $3$-parameter subgroup of $ \text{GL}(2,\mathbb{R})$.
    \item $\mathcal{G}$ with the subspace topology inherited from $\text{GL}(2,\mathbb{R})$ is homeomorphic to the $3$-dimensional real manifold $\mathcal{M} := \{(a,b,c) \in \mathbb{R}^3 : b > 0, \; a \neq 0\}$. Thus $\mathcal{G}$ has two connected components.
    \item Each element $A \in \mathcal{G}$ defines a homeomorphism $A: \mathcal{C} \rightarrow \mathcal{C}$.
\end{enumerate}
\end{lemma}

\begin{proof}
We first prove a result that we need for the proof of part (a). We define the set $\widetilde{\mathcal{C}} := \{(f,g) \in \mathcal{R}: g(x) > 0, \; \forall x \in \widehat{\mathcal{D}}\}$, and consider the matrix $\left( \begin{smallmatrix} 1 & 1 \\ -1 & 1 \end{smallmatrix} \right) =:B \in \text{GL}(2,\mathbb{R})$. We will show that $B$ defines a homeomorphism $B : \mathcal{C} \rightarrow \widetilde{\mathcal{C}}$. By Proposition~\ref{prop:homeo-R}, $B$ defines a homeomorphism $B: \mathcal{R} \rightarrow \mathcal{R}$, and so it suffices to show that $B(\mathcal{C}) = \widetilde{\mathcal{C}}$, as $B$ is an open map. So pick $(f,g) \in \mathcal{C}$, which implies $g(x) - f(x) > 0$ for all $x \in \widehat{\mathcal{D}}$. Then $B(f,g) = (g+f,g-f)$, and it follows that $B(\mathcal{C}) \subseteq \widetilde{\mathcal{C}}$. Next pick any $(\bar{f}, \bar{g}) \in \widetilde{\mathcal{C}}$, and define $(f,g) := \frac{1}{2}(\bar{f} - \bar{g}, \bar{f} + \bar{g})$. Then note that (i) $B(f,g) = (\bar{f},\bar{g})$, and (ii) $(f,g) \in \mathcal{C}$ as for any $x \in \widehat{\mathcal{D}}$, we have $f(x)-g(x) = -\bar{g}(x) < 0$. This shows that $\widetilde{\mathcal{C}} \subseteq B(\mathcal{C})$, and we have proved $\widetilde{\mathcal{C}} = B(\mathcal{C})$. We now return to the proof.

\noindent
(a) Choose any $A \in \mathcal{G}$, and define $\text{GL}(2,\mathbb{R}) \ni Q:= BAB^{-1}$, with $B$ defined as above. Then we claim that $Q$ leaves $\widetilde{\mathcal{C}}$ invariant, that is $Q(\bar{f}, \bar{g}) \in \widetilde{\mathcal{C}}$, for every $(\bar{f}, \bar{g}) \in \widetilde{\mathcal{C}}$. To see this, pick any $(\bar{f}, \bar{g}) \in \widetilde{\mathcal{C}}$. As $B: \mathcal{C} \rightarrow \widetilde{\mathcal{C}}$ is a homeomorphism, we have that $B^{-1}(\bar{f}, \bar{g}) \in \mathcal{C}$. Since $A \in \mathcal{G}$, we next conclude from the definition of $\mathcal{G}$ that $AB^{-1}(\bar{f}, \bar{g}) \in \mathcal{C}$. Finally this now implies that $BAB^{-1}(\bar{f}, \bar{g}) \in \widetilde{\mathcal{C}}$, which proves the claim. 

\noindent We now show that if $Q \in \text{GL}(2,\mathbb{R})$ leaves $\widetilde{\mathcal{C}}$ invariant, then $Q = \left( \begin{smallmatrix} a & c \\0 & b \end{smallmatrix} \right)$, for $a,b,c \in \mathbb{R}$, $a \neq 0$, and $b > 0$. Let us take such a matrix $Q$ and assume that $Q := \left( \begin{smallmatrix} Q_{11} & Q_{12} \\ Q_{21} & Q_{22} \end{smallmatrix} \right)$. For contradiction, suppose $Q_{21} \neq 0$. Take $\bar{g} \in C(\widehat{\mathcal{D}})$, defined as $\bar{g}(x)=1$, for all $x \in \widehat{\mathcal{D}}$. If we define $\bar{f}:= -\frac{Q_{22}}{Q_{21}} \bar{g}$, then we have (i) $(\bar{f},\bar{g}) \in \widetilde{\mathcal{C}}$, and (ii) $Q(\bar{f},\bar{g}) \not \in \widetilde{\mathcal{C}}$. This is a contradiction, and thus we must have $Q_{21} = 0$. Next, again for contradiction assume $Q_{22} \leq 0$. Take any $(\bar{f},\bar{g}) \in \widetilde{\mathcal{C}}$, and suppose $(f,g) := Q(\bar{f},\bar{g})$. Then $g = Q_{22} \bar{g}$, and thus $g(x) = Q_{22}\bar{g}(x) \leq 0$, for all $x \in \widehat{\mathcal{D}}$, which implies $(f,g) \not \in \widetilde{\mathcal{C}}$, giving the desired contradiction. Finally $Q \in \text{GL}(2,\mathbb{R})$ implies that $\det(Q) = Q_{11} Q_{22} \neq 0$, and thus we must have $Q_{11} \neq 0$.

Combining the arguments in the previous two paragraphs, we have thus proved that if $A \in \mathcal{G}$, then $A = B^{-1} Q B$, where $Q = \left( \begin{smallmatrix} a & c \\0 & b \end{smallmatrix} \right)$, for $a,b,c \in \mathbb{R}$, $a \neq 0$, and $b > 0$. The proof is finished by noting that $B^{-1} = \frac{1}{2} \left( \begin{smallmatrix} 1 & -1 \\1 & 1 \end{smallmatrix} \right)$.

\noindent
(b) Define the set $\mathcal{S} \subseteq \text{GL}(2,\mathbb{R})$ 
\begin{equation}\label{group_A}
\mathcal{S} := \left \{ \begin{pmatrix}
            1 & -1 \\
            1 & 1
        \end{pmatrix}
        \begin{pmatrix}
            a & c \\
            0 & b
        \end{pmatrix}
        \begin{pmatrix}
            1 & 1\\
            -1 & 1
        \end{pmatrix}: a,b,c \in \mathbb{R}, \; b > 0,\; a \neq 0 \right \}.
\end{equation}
It is easily checked that $\mathcal{S}$ is a subgroup of $\text{GL}(2,\mathbb{R})$, and by part (a) we have $\mathcal{G} \subseteq \mathcal{S}$. Now take $A \in \mathcal{S}$ given by Eq.~\eqref{eq:A-decomp}, and take any $(f,g) \in \mathcal{C}$. Then we have $(\bar{f}, \bar{g}) := A(f,g)$, where 
\begin{equation}
\bar{f} = (a + b - c)f + (a - b + c)g, \; \bar{g} = (a - b - c)f + (a + b + c)g.
\end{equation}
Then for any $x \in \widehat{\mathcal{D}}$, we have $\bar{f}(x) - \bar{g}(x) = 2b (f(x) - g(x)) < 0$, and thus $(\bar{f}, \bar{g}) \in \mathcal{C}$, or equivalently $A \in \mathcal{G}$. Thus we have proved $\mathcal{S} \subseteq \mathcal{G}$, and we deduce $\mathcal{S} = \mathcal{G}$. Thus $\mathcal{G}$ forms a $3$-parameter subgroup of $\text{GL}(2,\mathbb{R})$.

\noindent
(c) Define the map $\mu: \mathcal{M} \rightarrow \text{GL}(2,\mathbb{R})$, given by $\mathcal{M} \ni (a,b,c) \mapsto \left( \begin{smallmatrix} 1 & -1 \\ 1 & 1 \end{smallmatrix} \right) \left( \begin{smallmatrix} a & c \\ 0 & b \end{smallmatrix} \right) \left( \begin{smallmatrix} 1 & 1 \\ -1 & 1 \end{smallmatrix} \right)$. Then it is easily checked that $\mu$ is a smooth embedding \cite[Chapter~4]{lee2012smooth}. It then follows by \cite[Proposition~5.2]{lee2012smooth} that $\mathcal{M}$ is homeomorphic to $\mathcal{G}$ in the subspace topology. $\mathcal{M}$ has two connected components: $\{(a,b,c) \in \mathbb{R}^3: b > 0, a > 0\}$ and $\{(a,b,c) \in \mathbb{R}^3: b > 0, a < 0\}$. Thus $\mathcal{G}$ also has two connected components given by the images of the two connected components of $\mathcal{M}$ under the map $\mu$.

\noindent
(d) We just need to show $A(\mathcal{C}) = \mathcal{C}$, and then the result follows as $A: \mathcal{R} \rightarrow \mathcal{R}$ is a homeomorphism by Proposition~\ref{prop:homeo-R}. By definition of $\mathcal{G}$ we already know that $A(\mathcal{C}) \subseteq \mathcal{C}$.  For the reverse inclusion, pick any $(f,g) \in \mathcal{C}$, and we see that $(f,g) = A(A^{-1}(f,g))$. Now by part (b), $A^{-1} \in \mathcal{G}$ as $\mathcal{G}$ is a subgroup. Then $A^{-1}(f,g) \in \mathcal{C}$, and thus it follows that $\mathcal{C} \subseteq A(\mathcal{C})$.
\end{proof}

\begin{corollary}
\label{cor:group-eqiv}
If $A \in \mathcal{G}$, then one can decompose it as $A=BC$, where $B \in \mathcal{T}$ and $C \in \mathcal{H}$. Moreover, this decomposition is unique. We have $\mathcal{G} = \mathcal{T} \rtimes \mathcal{H}$, the semidirect product of $\mathcal{T}$ and $\mathcal{H}$.
\end{corollary}

\begin{proof}
Suppose $A \in \mathcal{G}$ is given by Eq.~\eqref{eq:A-decomp}, for some $a, b, c \in \mathbb{R}$, $b > 0$, and $a \neq 0$. Expanding we get
\begin{equation}
    A = 
    \begin{pmatrix}
        a + b - c  & a-b+c \\
        a - b -c & a+b+c
    \end{pmatrix}= 2b
    \begin{pmatrix}
        p & q-1 \\ p-1 & q
    \end{pmatrix},
\end{equation}
where $p = (a+b-c)/2b$, and $q  = (a+b+c)/2b$. Here $p+q = 1 + a/b \neq 1$. This proves the required decomposition $A=BC$, with the choices $B = 2b \left( \begin{smallmatrix} 1 & 0 \\ 0 & 1 \end{smallmatrix} \right)$, and $C = \left( \begin{smallmatrix} p & q-1 \\ p-1 & q \end{smallmatrix} \right)$. Uniqueness of the decomposition can be proved as follows: if $A = \bar{B} \bar{C}$ is another decomposition with $\bar{B} := \lambda \left( \begin{smallmatrix} 1 & 0 \\ 0 & 1 \end{smallmatrix} \right) \in \mathcal{T}$, and $\bar{C}:= \left( \begin{smallmatrix} \bar{p} & \bar{q}-1 \\ \bar{p}-1 & \bar{q} \end{smallmatrix} \right) \in \mathcal{H}$, then equating the four matrix entries gives us $\lambda = 2b$, $p=\bar{p}$, and $q = \bar{q}$.

\noindent From the definition of $\mathcal{G}$ above, it is clear that $\mathcal{T}, \mathcal{H} \subseteq \mathcal{G}$. Since $\mathcal{G}$ is a group by Lemma~\ref{lem:subgroup-G}(b), this implies $\langle \mathcal{T}, \mathcal{H} \rangle \subseteq \mathcal{G}$. Also, by the first part of this corollary, we have $\mathcal{G} \subseteq \langle \mathcal{T}, \mathcal{H} \rangle$. This proves $\mathcal{G} = \langle \mathcal{T}, \mathcal{H} \rangle$. Since $\mathcal{T}$ is a normal subgroup of $\mathcal{G}$, and $\mathcal{T}\cap \mathcal{H}=\{{\left( \begin{smallmatrix} 1 & 0 \\ 0 & 1 \end{smallmatrix} \right)}\}$, we can conclude that $\mathcal{G}=\mathcal{T}\rtimes\mathcal{H}$.
\end{proof}

\noindent Because of Lemma~\ref{lem:subgroup-G}, the group action $\digamma$ given by Eq.~\eqref{eq:group-action} restricts to the subgroup $\mathcal{G}$ acting on the subset $\mathcal{C}$, and we obtain a group action $\digamma_{\mathcal{C}}$ on $\mathcal{C}$
\begin{equation}
\label{eq:group-action-C}
\digamma_{\mathcal{C}}: \mathcal{G} \times \mathcal{C} \rightarrow \mathcal{C}, \;\; (A, (f,g)) \mapsto A(f,g).
\end{equation}
Note that the group action $\digamma_{\mathcal{C}}$ is \textit{fixed point free}: the set $\{(f,g) \in \mathcal{C}: A(f,g)=(f,g), \; \forall A \in \mathcal{G}\}$ is empty. To see this, suppose for contradiction that there exists $(f,g) \in \mathcal{C}$ such that $A(f,g)=(f,g)$ for all $A \in \mathcal{G}$. In particular, this implies that for every $\lambda > 0$, we must have $(\lambda f, \lambda g) = (f,g)$, from which it follows that $f=g=0$. However $(0,0) \not \in \mathcal{C}$ giving the desired contradiction.

\subsection{Free group actions.}
\label{ssec:free-group-action}

Our next goal is to find the non-trivial subgroups of $\mathcal{G}$ that act freely on $\mathcal{C}$. Recall the definition of free group action:
\begin{definition}
\label{def:free-action}
Given a group $G$ with identity element $e$, and a set $X$, a group action $\mu: G \times X \rightarrow X$ is said to act \textit{freely} on $X$ if for all $x \in X$, the condition $\mu(g,x)=x$ implies $g=e$. 
\end{definition}

\noindent It is easy to see that there indeed exist subgroups of $\mathcal{G}$ that do not act freely on $\mathcal{C}$. One such subgroup is $\mathcal{H}$ that we already encountered. To see this, take $A := \left( \begin{smallmatrix} 1 & q-1 \\ 0 & q \end{smallmatrix} \right) \in \mathcal{H}$. Then for any $(f,0) \in \mathcal{C}$, we see that $A(f,0)=(f,0)$. Since $A \neq I$, we conclude that $\mathcal{H}$ does not act freely on $\mathcal{C}$. But other subgroups of $\mathcal{G}$ indeed exist that act freely on $\mathcal{C}$, and we will now show that there are at least two such non-trivial subgroups $\mathcal{J}_1$ and $\mathcal{J}_2$, which are defined later in Eq.~\eqref{eq:J1-J2-set}. To do this, we start by proving the following lemma:

\begin{lemma}
\label{lem:free-action-helper}
Let $A \in \mathcal{G}$ be given by $A=\left( \begin{smallmatrix} 1 & -1 \\ 1 & 1 \end{smallmatrix} \right) \left( \begin{smallmatrix} a & c \\ 0 & b \end{smallmatrix} \right) \left( \begin{smallmatrix} 1 & 1 \\ -1 & 1 \end{smallmatrix} \right)$, for $a,b,c \in \mathbb{R}$, $a \neq 0$, and $b > 0$, and suppose $A \neq I$. Then there exists $(f,g) \in \mathcal{C}$ such that $A(f,g)=(f,g)$, if and only if $2b=1$ and $2a \neq 1$.
\end{lemma}

\begin{proof}
First note that $A=I$ if and only if $a = b = \frac{1}{2}$, and $c=0$. So we will assume that this is not the case in this proof. Now suppose $(f,g) \in \mathcal{C}$ such that $A(f,g)=(f,g)$. Then upon simplification, we have $\left( \begin{smallmatrix} 2a & 2c \\ 0 & 2b \end{smallmatrix} \right) (g+f, g-f) = (g+f, g-f)$. This gives us two equalities: (i) $(2a-1) (g+f) + 2c(g - f) = 0$, and (ii) $(2b-1)(g-f) = 0$. Since $(f,g) \in \mathcal{C}$, we have $g(x) - f(x) > 0$, for all $x \in \widehat{\mathcal{D}}$. Then condition (ii) implies $2b=1$. Now for contradiction if we assume $2a=1$, then condition (i) implies that $2c(g-f)=0$, or equivalently $c=0$. But this implies $A = I$, which is not allowed.

\noindent For the converse direction, assume $2b=1$ and $2a \neq 1$. Take $\bar{g} \in C(\widehat{\mathcal{D}})$ defined by $\bar{g}(x)=1$, for all $x \in \widehat{\mathcal{D}}$. Also define $\bar{f}:= \frac{2c}{1-2a} \bar{g}$, and $(f,g) := \left( \begin{smallmatrix} 1 & 1 \\ -1 & 1 \end{smallmatrix} \right)^{-1} (\bar{f}, \bar{g}) = \frac{1}{2}(\bar{f} - \bar{g}, \bar{f} + \bar{g})$. Then it is clear that $(f,g) \in \mathcal{C}$, as $g(x) - f(x) = \bar{g}(x) = 1$, for all $x \in \widehat{\mathcal{D}}$. Additionally we have
\begin{equation}
\begin{split}
A(f,g) &= \left( \begin{smallmatrix} 1 & -1 \\ 1 & 1 \end{smallmatrix} \right) \left( \begin{smallmatrix} a & c \\ 0 & b \end{smallmatrix} \right) \left( \begin{smallmatrix} 1 & 1 \\ -1 & 1 \end{smallmatrix} \right)(f,g) = \left( \begin{smallmatrix} 1 & -1 \\ 1 & 1 \end{smallmatrix} \right) \left( \begin{smallmatrix} a & c \\ 0 & b \end{smallmatrix} \right) (\bar{f},\bar{g}) = \tfrac{1}{2} \left( \begin{smallmatrix} 1 & -1 \\ 1 & 1 \end{smallmatrix} \right) \left( \begin{smallmatrix} 2a & 2c \\ 0 & 2b \end{smallmatrix} \right) (\bar{f},\bar{g}) \\
&= \tfrac{1}{2} \left( \begin{smallmatrix} 1 & -1 \\ 1 & 1 \end{smallmatrix} \right) \left( \begin{smallmatrix} 2a & 2c \\ 0 & 1 \end{smallmatrix} \right) (\bar{f},\bar{g}) = \tfrac{1}{2} \left( \begin{smallmatrix} 1 & -1 \\ 1 & 1 \end{smallmatrix} \right) (\bar{f},\bar{g}) = (f,g),
\end{split}
\end{equation}
which finishes the proof.
\end{proof}

Next define two $2$-parameter subsets $\mathcal{J}_1, \mathcal{J}_2 \subseteq \mathcal{G}$
\begin{equation}
\label{eq:J1-J2-set}
\begin{split}
\mathcal{J}_1 &:= \left \{ \begin{pmatrix}
            1 & -1 \\
            1 & 1
        \end{pmatrix}
        \begin{pmatrix}
            b & c \\
            0 & b
        \end{pmatrix}
        \begin{pmatrix}
            1 & 1\\
            -1 & 1
        \end{pmatrix}: b,c \in \mathbb{R}, b > 0 \right \} \\
\mathcal{J}_2 &:= \left \{ \begin{pmatrix}
            1 & -1 \\
            1 & 1
        \end{pmatrix}
        \begin{pmatrix}
            \frac{1}{2} & c \\
            0 & b
        \end{pmatrix}
        \begin{pmatrix}
            1 & 1\\
            -1 & 1
        \end{pmatrix}: b,c \in \mathbb{R}, b > 0 \right \},
\end{split}
\end{equation}
and it is easily checked that both $\mathcal{J}_1$ and $\mathcal{J}_2$ are subgroups of $\mathcal{G}$. The following properties can also be easily verified: (i) $\mathcal{J}_1 \cap \mathcal{J}_2 = \left \{ \tfrac{1}{2}\left( \begin{smallmatrix} 1 & -1 \\ 1 & 1 \end{smallmatrix} \right) \left( \begin{smallmatrix} 1 & c \\ 0 & 1 \end{smallmatrix} \right)   \left( \begin{smallmatrix} 1 & 1 \\ -1 & 1 \end{smallmatrix} \right): c \in \mathbb{R} \right \}$, (ii) $\mathcal{T} \cap \mathcal{J}_2 = \{I\}$, and (iii) $\mathcal{T} \subseteq \mathcal{J}_1$. We can now prove the main lemma of this section:

\begin{lemma}
\label{lem:free-acting-subgroups}
The subgroups $\mathcal{J}_1$ and $\mathcal{J}_2$ act freely on $\mathcal{C}$. Moreover these subgroups are maximal in the following sense:
\vspace*{-0.25cm}
\begin{enumerate}[(i)]
    \item There does not exist a subgroup $\widehat{\mathcal{G}} \neq \mathcal{J}_1$ such that $\mathcal{J}_1 \subseteq \widehat{\mathcal{G}} \subseteq \mathcal{G}$, with $\widehat{\mathcal{G}}$ acting freely on $\mathcal{C}$.
    \item There does not exist a subgroup $\widehat{\mathcal{G}} \neq \mathcal{J}_2$ such that $\mathcal{J}_2 \subseteq \widehat{\mathcal{G}} \subseteq \mathcal{G}$, with $\widehat{\mathcal{G}}$ acting freely on $\mathcal{C}$.
\end{enumerate}
\end{lemma}

\begin{proof}
To prove that $\mathcal{J}_1$ (resp. $\mathcal{J}_2$) acts freely on $\mathcal{C}$, we pick any $(f,g) \in \mathcal{C}$, and suppose that $A(f,g)=(f,g)$, for some $A \in \mathcal{J}_1$ (resp. $A \in \mathcal{J}_2$). By Lemma~\ref{lem:free-action-helper}, we conclude that $A=I$. We now prove the maximality conditions:

\noindent
(i) For contradiction assume such a $\widehat{\mathcal{G}}$ exists, and pick $A := \left( \begin{smallmatrix} 1 & -1 \\ 1 & 1 \end{smallmatrix} \right) \left( \begin{smallmatrix} a & c \\ 0 & b \end{smallmatrix} \right) \left( \begin{smallmatrix} 1 & 1 \\ -1 & 1 \end{smallmatrix} \right) \in \widehat{\mathcal{G}} \setminus \mathcal{J}_1$. Then we have $a \neq b$. Now choose $\widetilde{A} := \left( \begin{smallmatrix} 1 & -1 \\ 1 & 1 \end{smallmatrix} \right) \left( \begin{smallmatrix} 1/4b & c \\ 0 & 1/4b \end{smallmatrix} \right) \left( \begin{smallmatrix} 1 & 1 \\ -1 & 1 \end{smallmatrix} \right) \in \mathcal{J}_1$. Since $\widehat{\mathcal{G}}$ is a subgroup, $A \widetilde{A} \in \widehat{\mathcal{G}}$. But we also have $A \widetilde{A} = \left( \begin{smallmatrix} 1 & -1 \\ 1 & 1 \end{smallmatrix} \right) \left( \begin{smallmatrix} a/2b & \tilde{c} \\ 0 & 1/2 \end{smallmatrix} \right)
\left( \begin{smallmatrix} 1 & 1 \\ -1 & 1 \end{smallmatrix} \right)
$, where $\tilde{c} = 2ac + c/2b$. Now using $a \neq b$ we obtain $a/2b \neq \frac{1}{2}$, and so by applying Lemma~\ref{lem:free-action-helper}, we can conclude that there exists $(f,g) \in \mathcal{C}$ such that $A \widetilde{A}(f,g) = (f,g)$. Thus $\widehat{\mathcal{G}}$ does not act freely on $\mathcal{C}$, giving a contradiction.

\noindent
(ii) Again for contradiction assume such a $\widehat{\mathcal{G}}$ exists, and pick $A := \left( \begin{smallmatrix} 1 & -1 \\ 1 & 1 \end{smallmatrix} \right) \left( \begin{smallmatrix} a & c \\ 0 & b \end{smallmatrix} \right) \left( \begin{smallmatrix} 1 & 1 \\ -1 & 1 \end{smallmatrix} \right) \in \widehat{\mathcal{G}} \setminus \mathcal{J}_2$. Then we have $a \neq \frac{1}{2}$. Now choose $\widetilde{A} := \left( \begin{smallmatrix} 1 & -1 \\ 1 & 1 \end{smallmatrix} \right) \left( \begin{smallmatrix} 1/2 & c \\ 0 & 1/4b \end{smallmatrix} \right) \left( \begin{smallmatrix} 1 & 1 \\ -1 & 1 \end{smallmatrix} \right) \in \mathcal{J}_2$. Since $\widehat{\mathcal{G}}$ is a subgroup, $A \widetilde{A} \in \widehat{\mathcal{G}}$, and a simple computation also shows $A \widetilde{A} = \left( \begin{smallmatrix} 1 & -1 \\ 1 & 1 \end{smallmatrix} \right) \left( \begin{smallmatrix} a & \tilde{c} \\ 0 & 1/2 \end{smallmatrix} \right)
\left( \begin{smallmatrix} 1 & 1 \\ -1 & 1 \end{smallmatrix} \right)
$, where $\tilde{c}$ is the same as in (i). Since $a \neq \frac{1}{2}$, we can again conclude by Lemma~\ref{lem:free-action-helper}, that $\widehat{\mathcal{G}}$ does not act freely on $\mathcal{C}$, giving a contradiction.
\end{proof}

\subsection{Visualizing the various groups.}
\label{ssec:group-viz}

The group $\mathcal{G}$ is a three-parameter group defined in Eq.~\eqref{group_A}. The parameters~$(a,b,c)$ lie in $\mathbb{R}^3$, and although $c$ is unrestricted, $a\ne0$ and $b>0$. Define   $p(a,b):=1+\delta(a)-\text{sign}(b)$, where $\delta(x) = 1$ if $x = 1$, and zero otherwise, and $\text{sign}(x) = x / |x|$ if $x \neq 0$, and $0$ if $x=0$. It is easy to show that $p(a,b)$ vanishes only when $a\ne0$ and $b>0$. By considering smooth approximations of $\delta(a)$ and $\text{sign}(b)$, one can then visualise the group elements as those $(a,b)$ for which $p(a,b)$ vanishes. We show this in Figure~\ref{fig:group_elements}. Further, all subgroups we have noted here can be thought of as topological subspaces in this three-dimensional space: $(\mathbb{R}\setminus\{0\})\times \mathbb{R}^+ \times \mathbb{R}$. A pictorial summary of the various groups in this work is presented in Figure~\ref{fig:groups}.

\begin{figure}[t!]
    \centering
    \includegraphics[trim={0cm 4cm 0cm 5cm},clip,width=0.8\linewidth]{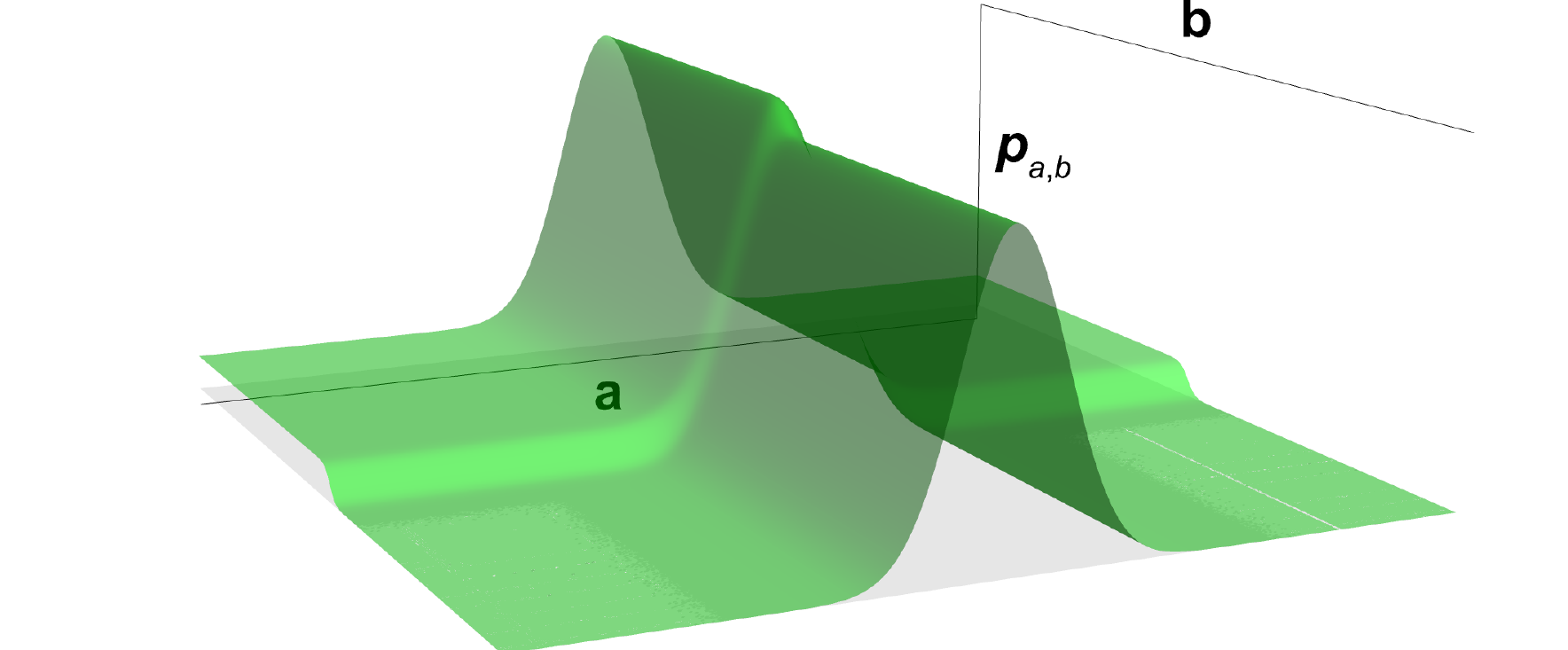}
    \caption{A visualisation of a smooth approximation of  $p(a,b):=1+\delta(a)-\text{sign}(b)$. Considering $(a,b)\in\mathbb{R}^2$ (fixing the parameter $c$), the group elements of $\mathcal{G}$ can be understood as zeros of the function $p(a,b)$ for a given $c$. When the smooth approximation approaches $p(a,b)$, the conditions $a\ne0$ and $b>0$ that define $\mathcal{G}$ are met strictly.}%
    \label{fig:group_elements}
\end{figure}

\begin{figure}[t!]
    \hspace{100pt}
{{\includegraphics[width=0.7\linewidth]{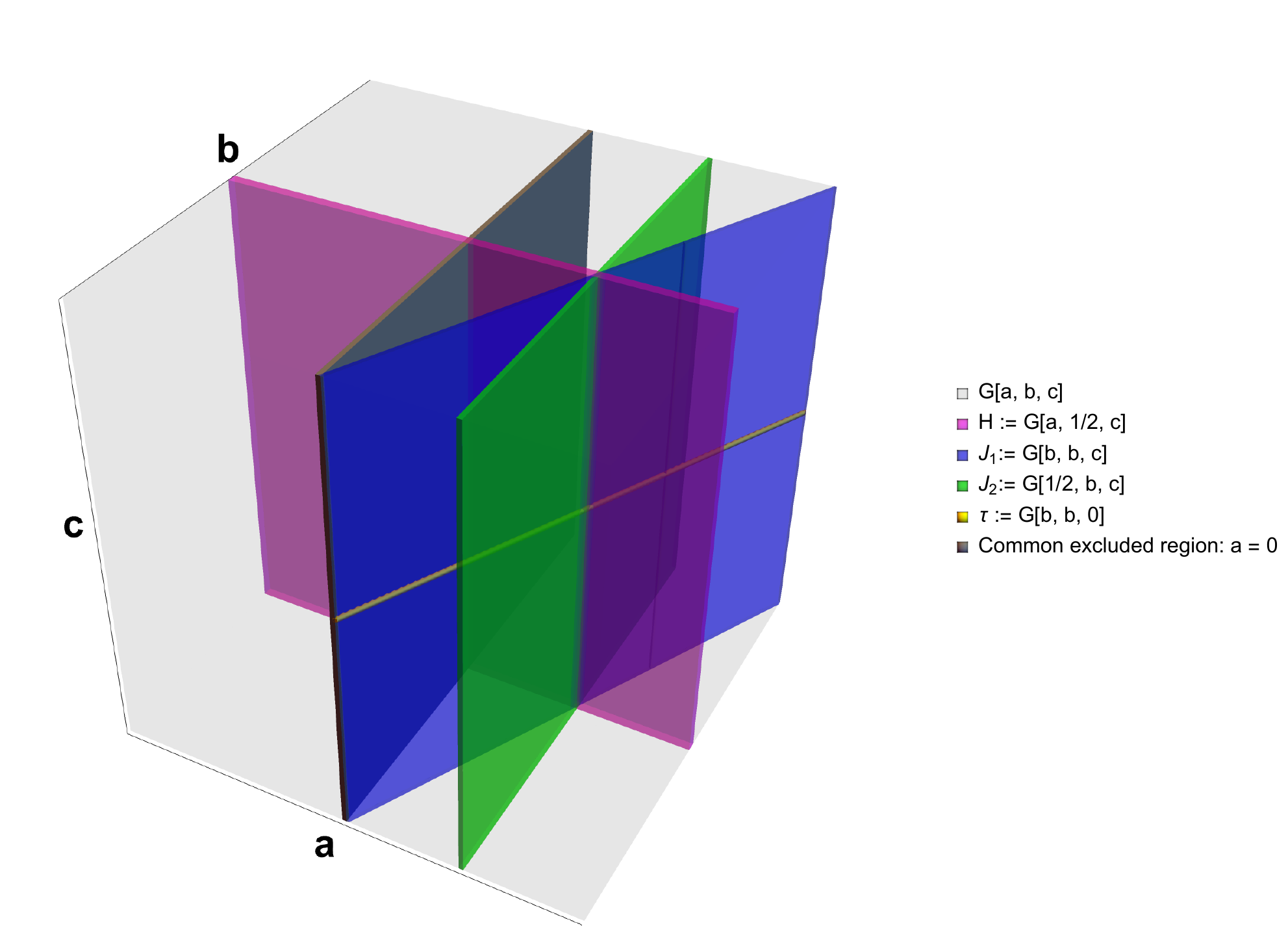} }}
    \caption{Various groups acting on the space of conjectures $\mathcal{C}$. The group $\mathcal{G}\subseteq \text{GL}(2,\mathbb{R})$ is a three parameter group with the parameters $(a,b,c)\in (\mathbb{R}\setminus\{0\})\times \mathbb{R}^+ \times \mathbb{R}$. The subgroups of $\mathcal{H},~\mathcal{T},~\mathcal{J}_1,~\mathcal{J}_2\subset\mathcal{G}$ are depicted in different colours. Some relationships between the various groups are obvious from the figure. For instance, the subgroup of dilations satisfies $\mathcal{T}\subseteq \mathcal{J}_1$ and $\mathcal{T} \cap \mathcal{J}_2 = \{I\}$, which can be deciphered from the plot. Some other relations cannot be inferred from the plot: for instance, $\mathcal{G} = \mathcal{T} \rtimes \mathcal{H}$, and $\mathcal{J}_1 \cap \mathcal{J}_2 = \left \{ \tfrac{1}{2}\left( \begin{smallmatrix} 1 & -1 \\ 1 & 1 \end{smallmatrix} \right) \left( \begin{smallmatrix} 1 & c \\ 0 & 1 \end{smallmatrix} \right)   \left( \begin{smallmatrix} 1 & 1 \\ -1 & 1 \end{smallmatrix} \right): c \in \mathbb{R} \right \}$. Further note that only $\mathcal{T}$, $\mathcal{J}_1$, and $\mathcal{J}_2$ in this collection act freely on the entire space $\mathcal{C}$.}%
    \label{fig:groups}%
\end{figure}

\subsection{Lie group structure.}
\label{ssec:lie-groups}

All the groups $\mathcal{T}, \mathcal{H}, \mathcal{G}, \mathcal{J}_1$, and $\mathcal{J}_2$ presented above turn out to be Lie groups. We prove this short result below, using the closed subgroup theorem \cite[Corollary~20.13]{lee2012smooth}:
\begin{lemma}[Lie group]
\label{lem:lie-groups}
Let $\mathcal{A}$ be any of the subgroups $\mathcal{T}, \mathcal{H}, \mathcal{G}, \mathcal{J}_1$, or $\mathcal{J}_2$. Then $\mathcal{A}$ is closed in $\text{GL}(2,\mathbb{R})$, $\mathcal{A}$ is an embedded submanifold of $\text{GL}(2,\mathbb{R})$, and $\mathcal{A}$ is an embedded Lie subgroup of $\text{GL}(2,\mathbb{R})$.
\end{lemma}

\begin{proof}
We will only prove that $\mathcal{A}$ is closed in $\text{GL}(2,\mathbb{R})$, and then the closed subgroup theorem implies the lemma as $\text{GL}(2,\mathbb{R})$ is a Lie group. Choose a sequence $A_i \rightarrow A$ in $\text{GL}(2,\mathbb{R})$, where $A_i \in \mathcal{A}$ for all $i$, and $A \in \text{GL}(2,\mathbb{R})$. Let us assume $A := \left(\begin{smallmatrix} A_{11} & A_{12} \\ A_{21} & A_{22} \end{smallmatrix}\right)$, for cases (i) and (ii) below, while for cases (iii)-(v) we assume $A = \left(\begin{smallmatrix} 1 & -1 \\ 1 & 1 \end{smallmatrix}\right)\left(\begin{smallmatrix} A_{11} & A_{12} \\ A_{21} & A_{22} \end{smallmatrix}\right) \left(\begin{smallmatrix} 1 & 1 \\ -1 & 1 \end{smallmatrix}\right)$. We must show that $A \in \mathcal{A}$. We do it case by case:
\begin{enumerate}[(i)]
    \item Case $\mathcal{A} = \mathcal{T}$: Let $A_i = \lambda_i \left(\begin{smallmatrix} 1 & 0 \\ 0 & 1 \end{smallmatrix}\right)$, for $\lambda_i > 0$ for all $i$. Since $A_i \rightarrow A$, we can conclude $\lambda_i \rightarrow A_{11}=A_{22} \geq 0$, and $A_{21}=A_{12}=0$. Also $A \in \text{GL}(2,\mathbb{R})$, so this implies $A_{11} \neq 0$, or equivalently $A \in \mathcal{T}$.
    \item Case $\mathcal{A} = \mathcal{H}$: For every $i$, take $A_i := \left(\begin{smallmatrix} p_i & q_i - 1 \\ p_i - 1 & q_i \end{smallmatrix}\right)$, where $p_i, q_i \in \mathbb{R}$, and $p_i + q_i \neq 1$. Using $A_i \rightarrow A$, we can conclude that  $p_i \rightarrow A_{11} = A_{21}+1$, and $q_i \rightarrow A_{22} = A_{12}+1$. Since $A \in  \text{GL}(2,\mathbb{R})$, we now conclude that $\det(A) = A_{11} + A_{22} - 1 \neq 0$. This implies $A \in \mathcal{H}$.
    \item Case $\mathcal{A} = \mathcal{G}$: For every $i$, take $A_i := \left(\begin{smallmatrix} 1 & -1 \\ 1 & 1 \end{smallmatrix}\right)\left(\begin{smallmatrix} a_i & c_i \\ 0 & b_i \end{smallmatrix}\right) \left(\begin{smallmatrix} 1 & 1 \\ -1 & 1 \end{smallmatrix}\right)$, with $a_i \neq 0$, $b_i > 0$. Then $A_i \rightarrow A$ implies $a_i \rightarrow A_{11}$, $b_i \rightarrow A_{22} \geq 0$, $A_{21}=0$, and $c_i \rightarrow A_{12}$. Since $A \in \text{GL}(2,\mathbb{R})$ we further have $\det(A) = A_{11} A_{22} \neq 0$, and thus we conclude $A_{11}, A_{22} \neq 0$. We have thus proved $A \in \mathcal{G}$.
    \item Case $\mathcal{A} = \mathcal{J}_1$: We take $A_i := \left(\begin{smallmatrix} 1 & -1 \\ 1 & 1 \end{smallmatrix}\right)\left(\begin{smallmatrix} b_i & c_i \\ 0 & b_i \end{smallmatrix}\right) \left(\begin{smallmatrix} 1 & 1 \\ -1 & 1 \end{smallmatrix}\right)$, with $b_i > 0$, for every $i$. Since $A_i \rightarrow A$, we may conclude $b_i \rightarrow A_{11} = A_{22} \geq 0$, $A_{21}=0$, and $c_i \rightarrow A_{12}$. Again $A \in \text{GL}(2,\mathbb{R})$ implies $\det(A) \neq 0$, and so $A_{11} \neq 0$. This implies $A \in \mathcal{J}_1$.
    \item Case $\mathcal{A} = \mathcal{J}_2$: Take $A_i := \left(\begin{smallmatrix} 1 & -1 \\ 1 & 1 \end{smallmatrix}\right)\left(\begin{smallmatrix} 1/2 & c_i \\ 0 & b_i \end{smallmatrix}\right) \left(\begin{smallmatrix} 1 & 1 \\ -1 & 1 \end{smallmatrix}\right)$, with $b_i > 0$, for every $i$. Then using $A_i \rightarrow A$ we conclude that $A_{11}= 1/2$, $A_{21}=0$, $c_i \rightarrow A_{12}$, and $b_i \rightarrow A_{22} \geq 0$. Since $A \in \text{GL}(2,\mathbb{R})$, this gives $A_{22} \neq 0$, and so we conclude that $A \in \mathcal{J}_2$.
\end{enumerate}
\end{proof}

%% file: computation.tex
\section{Machine guided search for conjectures}
\label{sec:computation}
In summary, we show that a metric space structure can be endowed to the space of algebraic inequalities. We identified the largest possible group of linear transformations that preserve such inequalities. We identify this group as the semidirect product of two distinct groups acting on the conjecture space. We show that conjecture space admits a free group action by the subgroup of dilations ($\mathcal{T}$), as well as two other groups $\mathcal{J}_1$ and $\mathcal{J}_2$ identified in the previous section. As such one can work on the corresponding quotient space. We denote our choice of the freely acting group by $G_f$. This gives a substantial computational advantage. We now proceed to describe our algorithm. 

\begin{algorithm}
\caption{Oracle}\label{alg:example}
\begin{algorithmic}[1]
\State \textbf{Inputs:} \textit{Mathematical features} $\{x_i\}_{i=1}^N$, \textit{function class} $\mathcal{F}_{d_1,d_2}$, and \textit{hyperparameters}: tolerance (tol), batch-size ($b$), maximum epochs (emax), learning rate ($\eta$).
\State $\theta \gets$ \text{random real vector}. 
\State \textbf{Parameterise:} $c_\theta:=(f_\theta,g_\theta):=\mathcal{P}({\theta}); f_\theta,g_\theta \in \mathcal{F}(K^{{G}_f}[x])$. 
\For{$i \gets 1$ to emax}
    \For{$j \gets 1$ to $b$}
     \If{$\mathcal{L}(c_\theta)\le {\text{tol}}$}
    \State return $\theta$
  \EndIf
\State $\theta\gets\theta-\eta~\mathcal{M}^{-1}_\theta ~\nabla \mathcal{L}(\theta)$
\State $\omega(c_\theta)\gets\frac{2}{b}\sum_{i\in \text{rand}}\text{sgn}(f_\theta(x_i)-g_\theta(x_i))$
\State $\mathcal{L}(c_\theta)\gets(1-\omega(c_\theta)^2)^2$
\EndFor
\EndFor
\State \textbf{Output:} $c_\theta=\mathcal{P}(\theta); f_\theta<g_\theta$. 
\end{algorithmic}
\end{algorithm}

\noindent In the above, the function class $\mathcal{F}_{d_1,d_2}$ is defined as the vector space spanned by functions with arguments that are symmetric polynomials in $x$ of degrees atmost $d_1$. Further such functions involve homogeneous functions (and hence scale invariant) of degree atmost $d_2$. As an example such combinations could include $sine,ln,exp,\ldots$ in variables upto feature degree $d_1$. For example, a combination could be $\sim\theta_1~f(x_{1,1}+x_{1,2})+\theta_2~g(x_{1,1} x_{1,2})$, where $d_1=2$ and $d_2=1$ and the function class $\mathcal{F}_{d_1(=2),d_2}$ contains the functions $f$ and $g$, which could be monomials, for example. The parameterisation $\mathcal{P}(\theta)$ describes this. We can understand $\mathcal{P}(\theta)$ as a machine learning architecture which functionally represents the tuple $(f_\theta,g_\theta)$; its inputs are $x\in\mathcal{\widehat{D}}$, and parameters $\theta$, and outputs are $(f_\theta(x),g_\theta(x))$. The quantity $\mathcal{M}_\theta$ is a Riemannian metric over the parameter space~$\Theta$ of this machine architecture, and $K^{{G}_f}[x]$ is the ring of invariants under $G_f$. 

\noindent In this work, we consider linear combinations and instead make the function class more interesting drawing from domain knowledge. The loss $\mathcal{L}$ is an invariant of the underlying group $\mathcal{G}$, since the function sgn is. More general losses are possible and can lead to different outcomes and different optimisation techniques. 

%% file: case_studies.tex
\section{Case studies}\label{sec:case_studies}
We now proceed to describe some conjectures obtained by querying the \textit{oracle}. Note that, although we found it is relatively easy to generate conjectures, being in a low dimensional setting, we should proceed with caution, when it comes to the actual veracity of the conjectures presented here. In what follows, although limited  sanity checks were conducted, we expect a number of conjectures to be false. We expect to return to these in future work. In the case of our examples of conjectures about the prime counting function, we focus on a singular conjecture, which we could prove, resulting in a new theorem.  

\subsection{Number theory.}
Consider the prime counting function $\pi$, which acts on a positive integer to return the number of primes less than equal to the number. This choice is inspired by Hardy-Littlewood's second conjecture~\cite{hardy1923some} which states that 
\begin{align}\label{conj:HL}
\vspace{-5pt}
    \pi(a+b)\le\pi(a)+\pi(b), \; 0 < a,b \in \mathbb{Z}.
\vspace{-15pt}
\end{align}
 This conjecture is from a century ago. Although the conjecture itself might be proven to be false, we found it might be worth investigating \textit{nearby} conjectures which might turn out to be true and provably true at best.
In effect this dictates the choice of the function classes $f_\theta$ and $g_\theta$. We specialise to the simple situation where these classes are linear in $\theta\in\Theta$. Following the set up laid out in Section~\ref{sec:computation}, we write down some conjectures obtained by querying the oracle in Algorithm~\ref{alg:example}. Although a large number of conjectures were found, we choose to present a handful of relatively simple conjectures in the current paper in Table~\ref{tab:example} for sake of brevity. 

\noindent Based on private communications with Fellows of London Institute of Mathematical Sciences, the first conjecture was partially proven and the following proof was provided to us by Thomas Fink, Yang-Hui He and Ilya Shkredov. 
\begin{theorem}
If $x,y$ are positive integers, then $\pi(x y) \geq \pi(x) + \pi(y),~~\forall \;x,y \ge 17 $. In addition, the inequality holds for all $2 \leq x,y < 17$.
\end{theorem}

\begin{proof}
By the Rosser-Schoenfeld formula \cite[Corollary~1]{rosser1962approximate}, we have for any $\alpha \geq 1.25506$
\begin{equation}
    \frac{x}{\log x} < \pi(x) <  \frac{\alpha x}{\log x}, \;\; \forall \; x \geq 17.
\end{equation}
Also note that for all $x,y \geq 17$ we have: (i) $\log x < \frac{x}{2 \alpha}$, and (ii) $\log x + \log y \leq \log x \log y$. Now choose $\alpha = 1.26$. Then for all $x,y \geq 17$, we have using the above facts
\begin{equation}
\begin{split}
    \pi(x) + \pi (y) & < \alpha \left( \frac{x}{\log x} + \frac{y}{\log y} \right) = \alpha \left( \frac{x \log y + y \log x}{\log x \log y} \right) < \left( \frac{\frac{xy}{2} + \frac{xy}{2}}{\log x \log y} \right) \\
    &= \left( \frac{xy}{\log x \log y} \right) \leq \left( \frac{xy}{\log x + \log y} \right) = \frac{xy}{\log (xy)} < \pi(xy).
\end{split}
\end{equation}
The case $2 \leq x,y < 17$ has been enumeratively checked using a computer.
\end{proof}
\newpage
{
\begin{table}[h]\label{prime_counting}
\centering
\begin{tabular}{|c|c|}
\hline
\textbf{\#} & \text{Conjecture ($c_\theta$)}  \\
\hline
\hline
1 & $\pi(a b) \ge \pi(a) + \pi(b)$  \\
\hline
2 & $\pi (a b)+\pi (a+b)+\pi (a)+\pi (b)\leq 2 a b$\\
\hline
 3 & $4 (\pi (a)+\pi (b))+\pi (a+b)\le 4 \pi (a b)$\\
\hline
 4 & $\pi (a b)+2 \pi (a+b)\geq \pi (a)+\pi (b)$\\
\hline
 5 & $\pi(a+ b) \le \pi(a) + \pi(b)$ \\
\hline
 6 & $\pi(a\ b\ c)\geq\pi(a)~\pi(b)~\pi(c)$ \\
 \hline
 7$^\dagger$ & $\pi(a\ b\ c\ldots)\geq\pi(a)~\pi(b)~\pi(c)\ldots$\\
\hline
 8 & $\pi(a+b+c)\leq \pi(a)+\pi(b)+\pi(c)$ \\
 \hline
  9$^\dagger$ & $\pi(a+b+c+\ldots)\leq \pi(a)+\pi(b)+\pi(c)+\ldots$ \\
\hline
10 & $\pi(a) \pi(b) \pi(c)\leq \sqrt{(\pi(a)+\pi(b)+\pi(c))^2+(\pi(a b c))^2}$\\
\hline
11$^{\dagger\star}$ & 
   ${(~(\Sigma_i\pi(a_i))^2+(\pi(\Pi_i a_i))^2)^{1/3}}\leq\Pi_i\pi( a_i)\leq{(~(\Sigma_i\pi(a_i))^2+(\pi(\Pi_i a_i))^2~)^{1/2}}$
   \\
\hline
12&$\pi (a b+ b c+ c a)^3\geq 2 \pi (a b c)^2+\pi (a b c)+\pi (a+b+c)^2$\\
\hline
13&$\pi (a+b+c)^2+\pi (a b c)\geq \pi (a b+ b c+ c a)^3+2 \pi (a b c)^2$\\
\hline
14 & $\pi (a b +b c + c a)^7\geq \pi (a+b+c)^7$\\
\hline
15 & $\pi \left(\alpha _1 \alpha _2 \alpha _3\right){}^2+ \pi \left(\alpha _1 \alpha _2+\alpha _2 \alpha _3+\alpha _3\alpha_1\right)\geq
    \pi \left(\alpha _1 \alpha _2+\alpha _2 \alpha _3+\alpha _3\alpha_1\right){}^3+ \pi \left(\alpha _1+\alpha _2+\alpha _3\right)$\\
    \hline
16&
    \tiny{
    $
    \pi \left(\alpha _1 \alpha _2 \alpha _3\right)^3+4 \pi \left(\alpha _1 \alpha _2 \alpha _3\right){}^2+4 \pi \left(\alpha _1+\alpha _2+\alpha _3\right)^3+3 
    +ldots
    \geq
     \pi \left(\alpha _1 \alpha _2+\ldots \right)$}\\
    \hline
17$^\circ$ & $\pi \left(\chi _2\right){}^3-\pi \left(\chi _3\right){}^2+\pi \left(\chi _3\right)\geq \pi \left(\chi _1\right)$\\
\hline
18&$\pi \left(\chi _1\right){}^3\geq \pi \left(\chi _2\right){}^3$\\
\hline
19&$5 \pi \left(\chi _2\right){}^3\geq  \pi \left(\chi _1\right)$\\
\hline
20&
$\pi \left(\chi _1\right)+\pi \left(\chi _3\right)<\pi \left(\chi _2\right)$\\
\hline
21 & $5 \pi \left(\pi \left(\chi _1\right)\right)+2 \pi \left(\pi \left(\chi _3\right)\right)\geq \pi \left(\pi \left(\chi _2\right)\right)$\\
\hline
22&$\pi \left(\pi \left(\chi_2\right)\right)\geq 2 \pi \left(\pi \left(\chi_3\right)\right)+3 \pi \left(\pi \left(\chi_1\right)\right)$\\
\hline
22 & $11 \left(\pi (a b)+\frac{a b}{\log (a b)}\right)>9 \pi (a+b)+\frac{9 (a+b)}{\log (a+b)}$\\
\hline
23 & $\pi \left(x+\sqrt{x}\right)\leq 3 \pi \left(x\right)+1$\\
\hline
24 $^{\dagger\dagger}$ & $\pi \left(x\right){}^2 > x^3+2x+2$ \\
\hline
25 & $\pi \left(x+\sqrt{x}\right) < 3\pi(x)$\\
\hline
 26 & $\pi \left(x+\sqrt{x}\right)<\frac{12}{5} \pi (x)+1$\\
\hline
\hline
\end{tabular}
\caption{\footnotesize{Table of conjectures for the prime counting function $\pi$. Conjectures derived from machine generated conjectures are made distinct with the marker $^{\dagger}$. In $^\star$, $a_i$ is to be read as $a_i+o(1)$. In $^\circ$, $\chi_i$ is the $i^\text{th}$ symmetric polynomial of $\mathbb{S}_3$. More precisely, $\chi=\{a+b+c,~ab+bc+ca,~abc\}$. Note that the conjecture $^{\dagger\dagger}$ can be interpreted as denoting the exterior points of an elliptic curve with coordinates $(x,\pi(x))$.}}
\label{tab:example}
\end{table}}
In addition to the proof of the theorem, we note the following subcases: (1) conjecture 1 holds asymptotically, (2) conjecture 1 holds when $a=b$. 
The proof of the theorem required knowledge of the Rosser-Schoenfeld formula \cite{rosser1962approximate}, however, subcases of the problem could be resolved using the asymptotic behaviour of the prime counting function alone, namely $\pi(x)\sim \frac{x}{\log x}$. We have not conducted such an exercise for all the conjectures listed. For conjecture 1 an analogous result proven for $a, b \ge \sqrt{53}$ in \cite{mincu2008properties}. Conjectures 2 and 7 can be seen to be true in the asymptotic limit. Conjecture 5 is HL's second conjecture~\cite{hardy1923some};~\eqref{conj:HL} while conjectures 8 and 9 are analogous to the HL conjecture. Conjectures 18 and 19 are trivially true.
\subsection{Finite simple groups.}
The completion of the classification of finite simple groups was an important advancement in mathematics aided by computers. 
In recent times, finite simple groups have been studied through the lens of machine learning~\cite{he2022learning,He_sharnoff_Jejjala_mishra}\footnote{A new theorem in this setting is expected to be reported in~\cite{He_sharnoff_Jejjala_mishra}.}.
In this subsection, we query the oracle for conjectures involving the generators of non-abelian finite simple groups and diameters of corresponding Cayley graphs with respect to a symmetric set of such generators. 
For simplicity, we will study non-abelian simple subgroups of $\mathcal{H}\subseteq\mathbb{S}_n$. As such we do not consider sporadic groups. 
All conjectures relate to this class of groups. Keeping in mind all finite simple groups are \textit{2-generated}, we consider the following properties of the group generators and associated Cayley graphs;
\begin{align}
    \mathcal{H}:=&\langle\sigma, \tau\rangle,~~~S=\{\sigma,\tau\}, ~~~\mathcal{D}:=\text{diam}(\text{Cay}(\mathcal{H},S\cup S^{-1})).\\
    \tau_1:=&tr(\sigma),~~~\tau_2:=tr(\tau),~~~~ \mathcal{O}_1:=o(\sigma),~\mathcal{O}_2:=o(\sigma).
\end{align}
One conjecture of interest is Babai's conjecture~\cite{babai1992diameter}, which upper bounds the diameter of the Cayley graph by a function of the order of the group, that is,
\begin{align}\label{babai}
    \mathcal{D}~\le~(\text{log}_2 |\mathcal{H}|)^c~,~\text{for some universal constant c}.
\end{align}
We search for \textit{nearby} conjectures, present some randomly selected  conjectures involving these quantities in Table~\ref{simple_groups_table}. 
\begin{table}[H]\label{simple_groups_table}
\centering
\begin{tabular}{|c|c|}
\hline
\textbf{\#} & \text{Conjecture ($c_\theta$)} \\
\hline
\hline
1 & $\mathcal{D}\ge(\tau _1+\tau _2)/2$ \\
\hline
2 & $\mathcal{D}\le(\mathcal{O}_1+\mathcal{O}_2)$\\
\hline
3 & $\mathcal{D}\ge (\mathcal{O}_1+\mathcal{O}_2)/4+{5 (\tau _1+\tau _2)}/{12}- 0.120225~\text{log}_2(G)$\\
\hline
4 & $\mathcal{D}\ge (4.32809 \log (G)+4\left(o_1+o_2\right)+5\left(\tau _1+\tau _2\right))/3$\\
\hline
\hline
\end{tabular}
\caption{Table of conjectures for finite simple groups.}
\end{table}
We bear in mind here that the data used to generate such conjectures correspond to a samples with $n\le9$. No further sanity checks were conducted and as such, these should be taken with a grain of salt. We also note that conjectures 3 and 4 in Table~\ref{simple_groups_table} aren't particularly aesthetically pleasing. However, we present this to highlight that such conjectures could in principle be churned out by the oracle. 